\documentclass[12pt]{article}

\usepackage[hmargin=1.7cm, vmargin={3cm,3cm}, a4paper]{geometry}\usepackage{amsmath,amssymb,amsthm,mathrsfs}
\usepackage{epsfig,epsf,subfigure,graphicx,graphics}
\usepackage{url,enumerate}
\graphicspath{{fig/}}
\usepackage{float}
\usepackage{parskip}
\usepackage{fancyhdr}
\usepackage[dvipsnames]{xcolor}
\usepackage{physics}

\DeclareMathAlphabet\mathrsfso      {U}{rsfso}{m}{n}
\usepackage{amssymb}\usepackage{mathtools}
\usepackage[toc,page]{appendix}
\usepackage{bbm}

\usepackage{lettrine}
\usepackage{hyperref}

\usepackage{algorithm}
\usepackage{algpseudocode}
\usepackage{algpseudocode}

\newtheorem{lemma}{Lemma}
\newtheorem*{lemma*}{Lemma}

\newtheorem{definition}{Definition}

\newtheorem{theorem}{Theorem}
\newtheorem*{theorem*}{Theorem 3}

\newtheorem{corollary}{Corollary}[theorem]

\begin{document}

\title{Limitations of Using Identical Distributions for Training and Testing When Learning Boolean Functions}

\author{Jordi Pérez-Guijarro}
\date{\small SPCOM Group, Universitat Politècnica de Catalunya, Barcelona, Spain}

\maketitle

\abstract{When the distributions of the training and test data do not coincide, the problem of understanding generalization becomes considerably more complex, prompting a variety of questions. Prior work has shown that, for some fixed learning methods, there are scenarios where training on a distribution different from the test distribution improves generalization. However, these results do not account for the possibility of choosing, for each training distribution, the optimal learning algorithm, leaving open whether the observed benefits stem from the mismatch itself or from suboptimality of the learner. In this work, we address this question. That is, we study whether it is always optimal for the training distribution to be identical to the test distribution when the learner is allowed to be optimally adapted to the training distribution. Surprisingly, when we restrict our attention to efficient learners, i.e., those running in polynomial time, and assume the existence of one-way functions, the answer is negative. In other words, matching the training and test distributions is not always optimal.}

\section{Introduction}

When analyzing theoretically supervised learning methods, assumptions about the problem are usually made. Typically, we assume that the training distribution coincides with the test distribution, or that there exists some true labeling rule for the given data. However, in real scenarios, some of these assumptions are not satisfied, raising the question of how this affects the performance of the learning methods.

Deviations of the test distribution from the training distribution usually degrade the performance of the methods used \cite{hand2006classifier,daume2006domain}. For this reason, alternative methods that provide some protection against such distributional shifts have been proposed. For instance, \cite{duchi2021learning,blanchet2019robust} study the use of a worst-case type of loss, where instead of minimizing the empirical risk, one minimizes the supremum of the risk over a set of distributions close to the training distribution. This approach ensures that small distributional shifts do not produce significant changes in the performance of the learning method. Other approaches, such as \cite{montasser2024transformation}, consider different types of shifts in the test distribution, where the test data has the form $(T(x),y)$, with $(x,y)$ distributed as in the training set, and $T$ being some function. The intuition here is that, for instance, when applying rotations to images in a classification task, the label should remain unchanged. Domain adaptation methods \cite{ganin2015unsupervised,fernando2013unsupervised,shimodaira2000improving,ben2006analysis,ben2010theory}, used mainly for covariate shift, i.e., when the labeling rule remains unchanged but only the marginal distribution of the input variable changes, make use of (usually unlabeled) data from the test distribution to adapt the model to the new distribution. 

Interestingly, as pointed out in \cite{gonzalez2015mismatched,canatar2021out}, shifts in the distributions are not always detrimental. In particular, these works show that for certain fixed learning methods, there exist scenarios where a mismatch between the training and test distributions can actually improve generalization. However, they do not consider selecting the optimal learning method for each training distribution. Thus, the observed advantages of mismatched distributions in those settings may simply arise from the use of suboptimal learners rather than from the mismatch itself.

In this work, we address the fundamental question: is it really optimal for the test distribution to be identical to the training distribution when the learner  (assumed to be efficient) is allowed to be optimally adapted to the training distribution? To investigate this, we use tools from computational learning theory. In particular, instead of adopting the traditional PAC-learning framework introduced by Valiant \cite{valiant1984theory}, we focus on a variant in which the efficient evaluation of the model is also considered \cite{gyurik2023exponential,perez2024classical}. Intuitively, one might think the answer to the question is trivially “yes.” However, we show that this is not generally the case, assuming that one-way functions exist, which highlights the complex nature of understanding generalization in this setting.

This paper is organized as follows. Section \ref{notation_and_definitions} introduces the necessary notation and definitions used throughout the paper. Next, in Section \ref{second_section}, the main results are presented, showing that even if a concept class can be learned under some distribution, it may not be learnable when the training distribution coincides with the test distribution. Finally, in Section \ref{conclusions}, the conclusions are presented.

\section{Notation and Definitions}\label{notation_and_definitions}

In this paper, we analyze the scenario of learning Boolean functions using a variant of the PAC-learning framework. To proceed, we introduce the necessary notation and definitions. First, in this setting, we are given the side information that the target function belongs to some concept class $\mathcal{C}_n$. Therefore, we aim to design a learning algorithm $\mathcal{A}$ that estimates any Boolean function $c_n:\{0,1\}^n \rightarrow \{0,1\}$ belonging to $\mathcal{C}_n$, given a training set $\mathcal{T}_{S}^{c_n} = \{(x_i, c_n(x_i))\}_{i=1}^{S}$, where $x_i \sim D^T_n$. Specifically, algorithm $\mathcal{A}$ has two inputs: $x$ and a training set $\mathcal{T}_S^{c_n}$, and outputs an estimate of $c_n(x)$. That is, it can be understood as the concatenation of, first, a learning method that selects a hypothesis $u_n : \{0,1\}^n \to \{0,1\}$, and then a method that evaluates it on the input $x$. The performance of algorithm $\mathcal{A}$ is measured by the risk, given by
    \begin{equation}\label{risk}
        R(\mathcal{A},c_n):=\mathbb{E}_{x\sim D^E_n,\,\mathcal{T}_{S}^{\,c_n},\, \mathcal{A}} \left[ \mathcal{L}(\mathcal{A}(x,\mathcal{T}_{S}^{\,c_n}),c_n(x))\right]
    \end{equation}
where $\mathcal{L}:\{0,1\}^2 \rightarrow \mathbb{R}^+$ denotes the loss function, and $D_n^E$ is the test distribution. Throughout this paper, we consider $\mathcal{L}(y,y'):= \mathbbm{1}\{y\neq y'\}$. In this case, the risk coincides with the error probability. Importantly, when an algorithm appears as a subscript of an expectation, as in \eqref{risk}, this indicates that the expectation is also taken with respect to the algorithm’s internal randomness.

In this setting, the learning problem is analyzed for a sequence of concept classes $\mathcal{C} = \{\mathcal{C}_n\}_{n \in \mathbb{N}}$, with an associated sequence of training distributions $D^T := \{D^T_n\}_{n \in \mathbb{N}}$ and test distributions $D^E := \{D^E_n\}_{n \in \mathbb{N}}$. With a slight abuse of terminology, we refer to both $\mathcal{C}_n$ and $\mathcal{C}$ as a concept class, and similarly, on some occasions, to a sequence of distributions $\{D_n\}_{n\in \mathbb{N}}$ as a distribution. With these preliminaries established, we can proceed to the formal definition of a learnable concept class $\mathcal{C}$.

\begin{definition}\label{average_case_learning}
    A concept class $\mathcal{C}$ is learnable for a sequence of test distributions $D^E$ if there exist an efficient randomized algorithm $\mathcal{A}$, a sequence of training distributions $D^T$, and a multivariate polynomial $p(a,b)$ such that algorithm $\mathcal{A}$, given input $x\in \{0,1\}^n$, a Boolean string of $m$ ones $1^m$, and a training set $\mathcal{T}_{p(n,m)}^{\,c_n}=\{(x_i,c_n(x_i))\}_{i=1}^{p(n,m)}$ with $x_i\sim D_n^T$, outputs $\mathcal{A}(x,\mathcal{T}_{p(n,m)}^{\,c_n},1^{m})\in \{0,1\}$ satisfying
        \begin{equation}\label{condition_definition_1}
            \mathbb{E}_{x\sim D_n^E,\,\mathcal{T}_{p(n,m)}^{\,c_n},\textnormal{ }\mathcal{A}}\left[ \mathbbm{1} \left\{\mathcal{A}(x,\mathcal{T}_{p(n,m)}^{\,c_n},1^{m})\neq c_n(x)\right\} \right]\leq\frac{1}{m}
        \end{equation}
    for all $n,m\in \mathbb{N}$ and $c_n \in \mathcal{C}_n $. The expected value is taken with respect to $x\sim D_n^E$, the training set $\mathcal{T}_{p(n,m)}^{c_n}$, and the internal randomness of algorithm $\mathcal{A}$.
\end{definition}

In this definition, the algorithm $\mathcal{A}$ takes an additional argument $1^m$, which represents a given precision requirement. As is common in computational complexity, we use $1^m$ instead of $m$ directly to ensure that any efficient algorithm runs in time $\mathrm{poly}(n, m)$ rather than $\mathrm{poly}(n, \log m)$. We refer to a sequence of training distributions $D^T$ as a sufficiently informative distribution for learning $\mathcal{C}$ with respect to $D^E$ if there exists an algorithm $\mathcal{A}$ and a polynomial $p(a,b)$ that satisfy the conditions specified in Definition \ref{average_case_learning}. The set of sufficiently informative distributions is denoted by $\mathrsfso{D}(\mathcal{C}, D^E)$.

We conclude by introducing the definitions of two complexity classes employed in deriving the results of this work: $\mathsf{P/poly}$ and $\mathsf{HeurBPP/samp}(D)$. In \cite{perez2024classical}, $\mathsf{HeurBPP/samp}(D)$ is defined as a class of distributional problems, i.e., pairs consisting of a language and a distribution. Since the distribution $D$ is fixed as a parameter of the class, we may, without loss of generality, treat it as a class of languages.

\begin{definition}
    $\mathsf{P/poly}$ is the class of languages decidable by a Turing Machine (TM) $M$ running in polynomial time with access to a sequence of strings $\{a_n\}_{n\in \mathbb{N}}$ of polynomial size. In other words,
    \begin{equation}
        M(x,a_{\ell(x)})=\mathbbm{1}\{x\in L\}
    \end{equation}
    for any binary string $x$, where $\ell(x)$ denotes the length of string $x$.
\end{definition}

\begin{definition} \cite{perez2024classical}
    $\mathsf{HeurBPP/samp}$\textnormal{(${D}$)} is the class of languages $L$ for which there exists a probabilistic Turing machine (TM) $M$ running in polynomial time and a polynomial $p(n,m)$, such that for the fixed sequence of distributions ${D}=\{D_n\}_{n\in \mathbb{N}}$, 
    \begin{equation}\label{equation_BPPsamp}
        \mathbb{P}_{x\sim D_n} \left(\mathbb{P}_{\mathcal{T}_{p(n,m)},\, M}\left(M(x,\mathcal{T}_{p(n,m)},1^m)=\mathbbm{1}\{x\in L\}\right)\geq \frac{2}{3} \right)\geq 1-\frac{1}{m}
    \end{equation}
    for any $n,m\in \mathbb{N}$, where $\mathcal{T}_{p(n,m)}=\{(x_i,\mathbbm{1}\{x_i\in L\})\}_{i=1}^{p(n,m)}$, and $x_i\sim D_n$. The external probability is taken with respect to $x\sim D_n$, and the internal probability with respect to the training set $\mathcal{T}_{p(n,m)}$, and the internal randomness of  TM $M$.
\end{definition}

Importantly, the condition given in equation \eqref{equation_BPPsamp} can be replaced by
\begin{equation}
    \mathbb{E}_{x\sim D_n,\mathcal{T}_{p(n,m)},\textnormal{ }M}\left[ \mathbbm{1} \left\{M(x,\mathcal{T}_{p(n,m)},1^{m})=\mathbbm{1}\{x\in L\}\right\} \right]\geq 1-\frac{1}{m}
\end{equation}
without modifying the definition (see Appendix A of \cite{perez2024classical}). This equivalence will be used throughout the proofs.

\section{Limitations of Identical Distributions for Training and Testing}\label{second_section}

Intuitively, one may think that to learn a concept class for the test distribution $D^E$, the best scenario is for the training set to be distributed according to $D^E$ as well. However, this intuition is false assuming that one-way functions exist, as shown in the following result.

\begin{theorem}\label{theorem_out_distribution}
    If one-way functions exist, then there exists a pair $(\mathcal{C}, D^E)$ such that $\mathrsfso{D}(\mathcal{C}, D^E) \neq \emptyset$, i.e., there exists a sufficiently informative training distribution $D^T$, and yet $D^E \notin \mathrsfso{D}(\mathcal{C}, D^E)$.
\end{theorem}

That is, if one-way functions exist, there are concept classes $\mathcal{C}$ that, although they can be learned for the test distribution $D^E$, cannot be learned when the training distribution is the same as the test distribution. The proof follows trivially by combining Lemma \ref{lemma_1}, Lemma \ref{lemma_2}, and Lemma \ref{lemma_3}.

\begin{lemma}\label{lemma_1}
    If there exists a sequence of distributions $D$ and a language $L$ such that $L\notin \mathsf{HeurBPP/samp}(D)$ and $L\in \mathsf{P/poly}$, then there exists a pair $(\mathcal{C}, D^E)$ such that $\mathrsfso{D}(\mathcal{C}, D^E) \neq \emptyset$ and yet $D^E \notin \mathrsfso{D}(\mathcal{C}, D^E)$.
\end{lemma}

\begin{proof}
    Let $\mathcal{C}_n$ be formed only by the concept $c_n(x)=\mathbbm{1}\{x \in L\}$. Then, by the definition of the class $\mathsf{HeurBPP/samp}(D)$, we have $D \notin \mathrsfso{D}(\mathcal{C},D)$. Next, since $L \in \mathsf{P/poly}$, there exists advice $\{a_{n}\}_{n\in \mathbb{N}}$ and a TM $M$ that, when combined, allows us to compute $c_n(x)$. For finitely many values of $n$ it is possible that $\ell(a_n)>2^n$. In these cases, we memorize the truth table, which only adds a constant to the sample complexity, absorbed into $p(n,m)$.
    
    We now design a distribution that encodes the advice $a_{n}$. Before describing the construction, we introduce some notation. For an integer $u \in \{0,1,\dots,2^n-1\}$, we denote by $u_b \in \{0,1\}^n$ its binary representation using $n$ bits. The encoding is as follows: the first bit of $a_n$ is encoded in the probability of obtaining the outcome $x=0_b$. If the first bit is $0$, this probability is $p_0$, otherwise, it is $2p_0$. The second bit is encoded analogously using the probability of the outcome $x=1_b$, and so on. For $x \notin \{0_b, 1_b, \cdots, (\ell(a_n)-1)_b\}$, the probability is exactly zero. That is, the support of the distribution has size $\ell(a_n)$. The value $p_0$ is determined by the restriction that the probabilities sum to $1$. Since each of the $\ell(a_n)$ outcomes has probability at most $2p_0$, we have $p_0 \geq 1/(2\ell(a_n))$, and hence $p_0 \geq 1/\mathrm{poly}(n)$. Therefore, since the support has polynomial size and $p_0 \geq 1/\mathrm{poly}(n)$, it follows that, for a sufficiently large polynomial in $n$, we can decode all bits correctly with constant probability. Hence, allowing the polynomial to depend on $m$ increases this probability to at least $1-1/m$. Finally, note that the all-zeros and all-ones advice strings yield indistinguishable encodings, since in both cases all outcomes in the support have the same probability. This can be resolved by encoding the additional bits $01$ at the end. Consequently, using a training set of size $\mathrm{poly}(n,m)$, we can decode the string $a_n$ with probability at least $1-1/m$, and afterwards use TM $M$ to compute $c_n(x)$. That is, $\mathscr{D}(\mathcal{C},D)$ is nonempty, although $D \notin \mathscr{D}(\mathcal{C},D)$.
\end{proof}

\begin{lemma}\label{lemma_2}
    If there exists some concept class $\mathcal{C}$ such that:

    \begin{enumerate}[$(i)$]
        \item At least one sequence of distributions $D$ satisfies $D\notin \mathrsfso{D}(\mathcal{C},D)$.
        
        \item All concepts $c_n^{(j)}\in \mathcal{C}_n$ can be efficiently computed by an algorithm, i.e., there exists an algorithm that, given $j$ and $x\in \{0,1\}^n$, outputs $c_n^{(j)}(x)$ in polynomial time.
    
        \item $|\mathcal{C}_n|\leq 2^{q(n)}$ for all $n\in\mathbb{N}$, where $q(n)$ is a polynomial. 
    \end{enumerate}
    \smallskip\smallskip\noindent
    then, for any sequence of distributions $D$ satisfying $D \notin \mathrsfso{D}(\mathcal{C}, D)$, there exists a language $L$ such that $L \notin \mathsf{HeurBPP/samp}(D)$ and $L\in \mathsf{P/poly}$.
\end{lemma}

\begin{proof}
    See Appendix \ref{proof_lemma_2}.
\end{proof}

\begin{lemma}\label{lemma_3}
    If one-way functions exist, then there exists a concept class $\mathcal{C}$ such that:

    \begin{enumerate}[$(i)$]

        \item The uniform sequence of distributions satisfies $U\notin \mathrsfso{D}(\mathcal{C},U)$.
        
        \item All concepts $c_n^{(j)}\in \mathcal{C}_n$ can be efficiently computed by an algorithm, i.e., there exists an algorithm that, given $j$ and $x\in \{0,1\}^n$, outputs $c_n^{(j)}(x)$ in polynomial time.
    
        \item $|\mathcal{C}_n|\leq 2^{q(n)}$ for all $n\in\mathbb{N}$, where $q(n)$ is a polynomial. 

        \item For any efficient algorithm $\mathcal{A}$ and any polynomials $p(n,m)$ and $r(n)$,

        \begin{equation}\label{equation_proof_lemma_3}
            \mathbb{P}_{c_n^{(j)}\in \mathcal{C}_n, x, \,z^{p(n,r(n))},\,\mathcal{A}}\left(\mathcal{A}(x,\mathcal{T}_{p(n,r(n))}^{c^{(j)}_n},1^{r(n)})=c_n^{(j)}(x) \right)\leq \frac{1}{2}+\epsilon(n)
        \end{equation}
        where $\epsilon(n)\rightarrow 0$ as $n\rightarrow \infty$.
    \end{enumerate}

\end{lemma}

\begin{proof}
    See Appendix \ref{proof_lemma_3}.
\end{proof}

Note that the result of Theorem \ref{theorem_out_distribution} may appear somewhat artificial, since we are allowing the use of non-uniform sequences of distributions for which there clearly does not exist an efficient algorithm to sample from. We now show that even if we restrict the analysis to sequences of distributions for which there exists an efficient algorithm $G$ that, on input $1^n$, generates samples from $D_n$, the same result still holds. Let $\mathrsfso{D}^{\text{eff}}(\mathcal{C}, D^E)$ denote the set of sequences of distributions that are sufficiently informative for the pair $(\mathcal{C}, D^E)$ and, additionally, can be efficiently sampled. With this notation introduced, we can now state the second result.

\begin{theorem}\label{theorem_one_way_eff} 
    If one-way functions exist, then there exists a pair $(\mathcal{C}, D^E)$ such that $D^E$ is efficiently samplable, and $\mathrsfso{D}^{\mathrm{eff}}(\mathcal{C}, D^E) \neq \emptyset$, and yet $D^E \notin \mathrsfso{D}^{\mathrm{eff}}(\mathcal{C}, D^E)$.
\end{theorem}
\begin{proof}
    Under the assumption that one-way functions exist, and following the construction of the concept class in Lemma \ref{lemma_3}, there exists a concept class $\mathcal{C}$ such that:

    \begin{enumerate}[$(i)$]

        \item The uniform sequence of distributions satisfies $U\notin \mathrsfso{D}(\mathcal{C},U)$.
        
        \item All concepts $c_n^{(j)}\in \mathcal{C}_n$ can be efficiently computed by an algorithm, i.e., there exists an algorithm that, given $j$ and $x\in \{0,1\}^n$, outputs $c_n^{(j)}(x)$ in polynomial time.
    
        \item $|\mathcal{C}_n|\leq 2^{n}$ for all $n\in\mathbb{N}$.

        \item For any efficient algorithm $\mathcal{A}$ and any polynomials $p(n,m)$ and $r(n)$,

        \begin{equation}
            \mathbb{P}_{c_n^{(j)}\in \mathcal{C}_n, x, \,z^{p(n,r(n))},\,\mathcal{A}}\left(\mathcal{A}(x,\mathcal{T}_{p(n,r(n))}^{c^{(j)}_n},1^{r(n)})=c_n^{(j)}(x) \right)\leq \frac{1}{2}+\epsilon(n)
        \end{equation}
        where $\epsilon(n)\rightarrow 0$ as $n\rightarrow \infty$.
    \end{enumerate}
    In this proof, we modify the concept class $\mathcal{C}$ by altering each concept $c_n^{(j)}$ only on the values $x \in \{ 0_b , 1_b, \ldots, (n-1)_b\}$ where $a_b$ denotes the binary sequence that represents the integer $a$ in binary using $n$ bits. The modified concept on those values is defined as:
    \begin{equation}
        \tilde{c}_n^{(j)}(x_b)=[j]_{x+1}
    \end{equation}
    where $[j]_u$ denotes the $u^{th}$ bit of $j$. Let $\tilde{\mathcal{C}}_n$ denote this modified concept class. Clearly, the uniform distribution over $\{0_b, 1_b, \ldots, (n-1)_b\}$ is sufficiently informative for the concept class $\tilde{\mathcal{C}}_n$, and it is also efficiently samplable. To see this, note that by taking $\mathrm{poly}(n)$ samples, for a sufficiently large polynomial, we observe, with high probability, every value in the support of the distribution. Thus, we learn the value of $j$, which, by condition $(ii)$, allows us to compute the function efficiently. That is, $\mathrsfso{D}^{\mathrm{eff}}(\tilde{\mathcal{C}}, U) \neq \emptyset.$

Next, let us denote the random variable $V^{(j)}$ as the pair $(X, c_n^{(j)}(X))$ and the random variable $\tilde{V}^{(j)}$ as the pair $(X, \tilde{c}_n^{(j)}(X))$, where $X \sim U_n$. The total variation distance satisfies
    \begin{align}
        d_{TV}(V^{(j)},\tilde{V}^{(j)}) &=\frac{1}{2} \sum_{x\in \{0,1\}^n, \, y \in \{0,1\}} \left|\mathbb{P}_{(x,y)\sim V^{(j)}}((x,y))-\mathbb{P}_{(x,y)\sim \tilde{V}^{(j)}}((x,y)) \right| \nonumber \\ &= \frac{1}{2} \sum_{\substack{x\in \{ 0_b , 1_b, \ldots, (n-1)_b\}\\ y \in \{0,1\}}} \left|\mathbb{P}(x)\mathbbm{1}\{y=c_n^{(j)}(x)\} -\mathbb{P}(x)\mathbbm{1}\{y=\tilde{c}_n^{(j)}(x)\} \right| \nonumber \\ &= \sum_{x\in \{ 0_b , 1_b, \ldots, (n-1)_b\}} \mathbb{P}(x) \mathbbm{1}\{c_n^{(j)}(x)\neq \tilde{c}_n^{(j)}(x)\}  \leq \frac{n}{2^n}
    \end{align}
    Consequently, in order to distinguish between the two random variables, an exponential number of samples is required. Specifically, to distinguish between these distributions with an average probability of error $\epsilon$ (assuming a prior of $1/2$ for each random variable), we need at least $N \geq \frac{1-2 \epsilon}{d_{TV}\bigl(V^{(j)}, \tilde{V}^{(j)}\bigr)}$ samples (Theorem 13.1.1 in \cite{lehmann2005testing}). This implies that for any polynomial $q(n)$, and for sufficiently large $n$, 
    \begin{align}\label{difference_distributions}
        &\Big |\mathbb{P}_{c_n^{(j)}\in \mathcal{C}_n, x, \,z^{p(n,r(n))},\,\mathcal{A}}\left(\mathcal{A}(x,\mathcal{T}_{p(n,r(n))}^{c^{(j)}_n},1^{r(n)})=c_n^{(j)}(x) \right) \nonumber \\ & \hspace{3cm}-\mathbb{P}_{\tilde{c}_n^{(j)}\in \tilde{\mathcal{C}}_n, x, \,z^{p(n,r(n))},\,\mathcal{A}}\left(\mathcal{A}(x,\mathcal{T}_{p(n,r(n))}^{\tilde{c}^{(j)}_n},1^{r(n)})=\tilde{c}_n^{(j)}(x) \right) \Big| \leq \frac{1}{q(n)}
    \end{align}
    as otherwise, we could distinguish between the two random variables using only a polynomial number of samples. Therefore, by combining \eqref{difference_distributions} and \eqref{equation_proof_lemma_3}, 

    \begin{equation}    
        \mathbb{P}_{\tilde{c}_n^{(j)}\in \tilde{\mathcal{C}}_n, x, \,z^{p(n,r(n))},\,\mathcal{A}}\left(\mathcal{A}(x,\mathcal{T}_{p(n,r(n))}^{\tilde{c}_n^{(j)}},1^{r(n)})=\tilde{c}_n^{(j)}(x) \right)\leq \frac{1}{2}+\epsilon'(n)
    \end{equation}
    where $\epsilon'(n)\rightarrow 0$ as $n\rightarrow \infty$. That is, the uniform sequence of distributions $U$ is not sufficiently informative, i.e., $U \notin \mathrsfso{D}^{\mathrm{eff}}(\tilde{\mathcal{C}}, U)$.
\end{proof}

That is, even when we consider only efficiently samplable distributions, the same effect is observed. Matching the training and test distributions is not always optimal for generalization, even when the learner is chosen optimally among efficient algorithms.

Finally, note that the results in this section rely on the restriction to efficient learners, since concept classes satisfying  $|\mathcal{C}_n| \leq 2^{q(n)}$ are information-theoretically learnable using $\mathrm{poly}(n,m)$ samples. This follows from the fact that the VC dimension of each such concept class is upper bounded by $\log|\mathcal{C}_n|= \mathrm{poly}(n)$. Therefore, its sample complexity \cite{hanneke2016optimal,ehrenfeucht1989general} is given by $\Theta\left(  \frac{d+\ln(1/\delta)}{\epsilon}\right)$, where $d$ denotes the VC dimension.

\section{Conclusions}\label{conclusions}

In summary, we have shown that, under the assumption that one-way functions exist, having the training distribution equal to the test distribution is not always optimal, even when the efficient learner is allowed to be optimally adapted to the training distribution. Furthermore, we have shown that this effect appears even when efficiently samplable sequences of distributions are considered. However, to achieve this latter result, the constructed example relies on the fact that some pairs $(x, c_n(x))$ are much more informative than others. As a future direction, it would be interesting to determine whether more natural examples exhibiting this property exist. Another open question is whether the existence of one-way functions is not only sufficient but also necessary for these results.

\bibliographystyle{ieeetr}
\bibliography{bibliography.bib}

\begin{thebibliography}{10}

\bibitem{hand2006classifier}
D.~J. Hand, ``{Classifier Technology and the Illusion of Progress},'' {\em Statistical Science}, vol.~21, no.~1, pp.~1 -- 14, 2006.

\bibitem{daume2006domain}
H.~Daume~III and D.~Marcu, ``Domain adaptation for statistical classifiers,'' {\em Journal of artificial Intelligence research}, vol.~26, pp.~101--126, 2006.

\bibitem{duchi2021learning}
J.~C. Duchi and H.~Namkoong, ``Learning models with uniform performance via distributionally robust optimization,'' {\em The Annals of Statistics}, vol.~49, no.~3, pp.~1378--1406, 2021.

\bibitem{blanchet2019robust}
J.~Blanchet, Y.~Kang, and K.~Murthy, ``Robust wasserstein profile inference and applications to machine learning,'' {\em Journal of Applied Probability}, vol.~56, no.~3, pp.~830--857, 2019.

\bibitem{montasser2024transformation}
O.~Montasser, H.~Shao, and E.~Abbe, ``Transformation-invariant learning and theoretical guarantees for ood generalization,'' {\em Advances in Neural Information Processing Systems}, vol.~37, pp.~108649--108673, 2024.

\bibitem{ganin2015unsupervised}
Y.~Ganin and V.~Lempitsky, ``Unsupervised domain adaptation by backpropagation,'' in {\em International conference on machine learning}, pp.~1180--1189, PMLR, 2015.

\bibitem{fernando2013unsupervised}
B.~Fernando, A.~Habrard, M.~Sebban, and T.~Tuytelaars, ``Unsupervised visual domain adaptation using subspace alignment,'' in {\em Proceedings of the IEEE international conference on computer vision}, pp.~2960--2967, 2013.

\bibitem{shimodaira2000improving}
H.~Shimodaira, ``Improving predictive inference under covariate shift by weighting the log-likelihood function,'' {\em Journal of statistical planning and inference}, vol.~90, no.~2, pp.~227--244, 2000.

\bibitem{ben2006analysis}
S.~Ben-David, J.~Blitzer, K.~Crammer, and F.~Pereira, ``Analysis of representations for domain adaptation,'' {\em Advances in neural information processing systems}, vol.~19, 2006.

\bibitem{ben2010theory}
S.~Ben-David, J.~Blitzer, K.~Crammer, A.~Kulesza, F.~Pereira, and J.~W. Vaughan, ``A theory of learning from different domains,'' {\em Machine learning}, vol.~79, no.~1, pp.~151--175, 2010.

\bibitem{gonzalez2015mismatched}
C.~R. Gonz{\'a}lez and Y.~S. Abu-Mostafa, ``Mismatched training and test distributions can outperform matched ones,'' {\em Neural computation}, vol.~27, no.~2, pp.~365--387, 2015.

\bibitem{canatar2021out}
A.~Canatar, B.~Bordelon, and C.~Pehlevan, ``Out-of-distribution generalization in kernel regression,'' {\em Advances in Neural Information Processing Systems}, vol.~34, pp.~12600--12612, 2021.

\bibitem{valiant1984theory}
L.~G. Valiant, ``A theory of the learnable,'' {\em Communications of the ACM}, vol.~27, no.~11, pp.~1134--1142, 1984.

\bibitem{gyurik2023exponential}
C.~Gyurik and V.~Dunjko, ``Exponential separations between classical and quantum learners,'' {\em arXiv preprint arXiv:2306.16028}, 2023.

\bibitem{perez2024classical}
J.~P{\'e}rez-Guijarro, ``On classical advice, sampling advice and complexity assumptions for learning separations,'' {\em arXiv preprint arXiv:2408.13880}, 2024.

\bibitem{lehmann2005testing}
E.~L. Lehmann and J.~P. Romano, {\em Testing statistical hypotheses}.
\newblock Springer, 2005.

\bibitem{hanneke2016optimal}
S.~Hanneke, ``The optimal sample complexity of pac learning,'' {\em Journal of Machine Learning Research}, vol.~17, no.~38, pp.~1--15, 2016.

\bibitem{ehrenfeucht1989general}
A.~Ehrenfeucht, D.~Haussler, M.~Kearns, and L.~Valiant, ``A general lower bound on the number of examples needed for learning,'' {\em Information and Computation}, vol.~82, no.~3, pp.~247--261, 1989.

\bibitem{goldreich1986construct}
O.~Goldreich, S.~Goldwasser, and S.~Micali, ``How to construct random functions,'' {\em Journal of the ACM (JACM)}, vol.~33, no.~4, pp.~792--807, 1986.

\end{thebibliography}

\begin{appendices} 

\section{Proof Lemma \ref{lemma_2}}\label{proof_lemma_2}

 In this appendix, we provide the proof of Lemma \ref{lemma_2}.

\begin{proof}
    The proof of this lemma is analogous to that of Lemma 1 in \cite{perez2024classical}, and follows by a diagonalization argument. The idea is to transform the concept class $\mathcal{C}$ into a language of the form
    \begin{equation}
        L(\mathcal{C}) := \{ x : c_{\ell(x)}^{(g(\ell(x)))}(x) = 1 \}.
    \end{equation}
    For the definition of the function $g(n)$, we first introduce some notation. Specifically, given an algorithm $\mathcal{A}$, a sequence of distributions $D$, and a polynomial $p(n,m)$, we define the set $\mathcal{E}_n(\mathcal{A},D,p)\subseteq \mathcal{C}_n$ as the set of concepts $c_n^{(j)}$  that satisfy
    \begin{equation}
            \mathbb{P}_{x\sim D_n}\left( \mathbb{P}_{\mathcal{T}_{p(n,m)}^{c_n^{(j)}},\,\mathcal{A}}\left(\mathcal{A}(x,\mathcal{T}_{p(n,m)}^{c_n^{(j)}},1^{m})=c_n^{(j)}(x)\right)\geq \frac{2}{3}\right)< 1-\frac{1}{m}
    \end{equation}   
    for some $m\in \mathbb{N}$, where the samples of the training set are distributed according to $D$. Therefore, by taking $D$ such that $D \notin \mathrsfso{D}(\mathcal{C},D)$, which exists by condition $(i)$, any pair $(\mathcal{A},p)$ satisfies $\mathcal{E}_n(\mathcal{A},D,p) \neq \emptyset$ for infinitely many $n \in \mathbb{N}$. The set of values of $n\in \mathbb{N}$ for which this occurs is denoted by
    \begin{equation}
        {\mathcal{N}}(\mathcal{A},D,p):=\{n\in \mathbb{N}: \mathcal{E}_n(\mathcal{A},D,p)\neq \emptyset\}
    \end{equation}
    Note that if $(\mathcal{A},p)$ fails for only finitely many values of $n$, then there exists a maximum such value, denoted by $n_{\max}$. In this case, we can take $p(n,m)+c(n_{\max},m)$ samples, for a sufficiently large factor $c(n_{\max},m)$, and define a new algorithm $\mathcal{A}'$ that uses an exhaustive search for $n\leq n_{\max}$ and runs $\mathcal{A}$ otherwise. This would make the class learnable, leading to a contradiction.
    
    Let $B_{\mathrm{alg}}$ denote a bijection between $\mathbb{N}$ and the set of efficient algorithms, and $B_{\mathrm{poly}}$ represent a bijection between $\mathbb{N}$ and the set of polynomials $\{c\,n^{k_1} m^{k_2}:(c,k_1,k_2)\in\mathbb{N}^3\}$. Now, we can define function $g(n)$:
    \begin{equation}
         g(n):=\left\{\begin{matrix}1 \text{ if }n\notin\{n_i\}_{i=1}^{\infty}\\ \min \{j: c_n^{(j)}\in \mathcal{E}_n\left(B_{\mathrm{alg}}(s_{i^*(n)}),D,B_{\mathrm{poly}}(z_{i^*(n)})\right)\}\text{ if }n\in\{n_i\}_{i=1}^{\infty}
    \end{matrix}\right.
    \end{equation}
    where sequences $s_i\in \mathbb{N}$ and $z_i\in \mathbb{N}$ satisfy that any point $(a,b)\in \mathbb{N}^2$ appears at least once in the sequence $\{(s_i,z_i)\}_{i\in \mathbb{N}}$. The sequence $\{n_i\}_{i\in \mathbb{N}}$ is defined as follows:
    \begin{equation}
    n_1:=\min\{{\mathcal{N}}\left(B_{\mathrm{alg}}(s_{1}),D,B_{\mathrm{poly}}(z_{1})\right) \},
    \end{equation}
    and 
    \begin{equation}
    n_i=\min\{{\mathcal{N}}\left(B_{\mathrm{alg}}(s_{i}),D,B_{\mathrm{poly}}(z_{i})\right) \backslash\{n_1,n_2,\cdots,n_{i-1}\} \}
    \end{equation}
    Finally, $i^*(n)$ denotes the index $i$ such that $n_i=n$.

    Next, we proceed to prove that $L(\mathcal{C}) \notin \,$$\mathsf{HeurBPP/samp}$($D$), but $L(\mathcal{C})\in \,\mathsf{P/poly}$. It is straightforward to verify that $L(\mathcal{C})\in\,\mathsf{P/poly}$. Given $g(n)$, which can be represented with a polynomial in $n$ number of bits (condition $(iii)$), we can efficiently compute $c_n^{(g(n))}(x)$ (condition $(ii)$). To establish that $L(\mathcal{C}) \notin \mathsf{HeurBPP/samp}(D)$, we employ a proof by contradiction. That is, we assume that $L(\mathcal{C}) \in \mathsf{HeurBPP/samp}(D)$, which implies the existence of a pair $(\mathcal{A},p)$ such that
    \begin{equation}
            \mathbb{P}_{x\sim D_n}\left( \mathbb{P}_{\mathcal{T}_{p(n,m)}^{c_n^{(g(n))}},\,\mathcal{A}}\left(\mathcal{A}(x,\mathcal{T}_{p(n,m)}^{c_n^{(g(n))}},1^{m})=c_n^{(g(n))}(x)\right)\geq \frac{2}{3}\right)\geq 1-\frac{1}{m}
    \end{equation}  
    for all $n,m\in \mathbb{N}$. Since each point $(a,b) \in \mathbb{N}^2$ appears at least once in the sequence $\{(s_i,z_i)\}_{i \in \mathbb{N}}$, there exists an index $i$ such that $(s_i,z_i) = (B_{\mathrm{alg}}^{-1}(\mathcal{A}), B_{\mathrm{poly}}^{-1}(p))$. Therefore, for $n = n_i$, it follows from the definition of $g(n)$ that
    \begin{equation}
            \mathbb{P}_{x\sim D_{n_i}}\left( \mathbb{P}_{\mathcal{T}_{p(n_i,m)}^{c_{n_i}^{(g(n_i))}},\,\mathcal{A}}\left(\mathcal{A}(x,\mathcal{T}_{p(n_i,m)}^{c_{n_i}^{(g(n_i))}},1^{m})=c_{n_i}^{(g(n_i))}(x)\right)\geq \frac{2}{3}\right)< 1-\frac{1}{m}
    \end{equation}  
    for some $m\in \mathbb{N}$. Consequently, the assumption results in a contradiction, showing that $L(\mathcal{C}) \notin \mathsf{HeurBPP/samp}(D)$.
\end{proof}

\section{Proof Lemma \ref{lemma_3}}\label{proof_lemma_3}

    To begin the proof, we use Theorem 3 \cite{goldreich1986construct}, which can be restated as:

    \begin{theorem}
        If one-way functions exist, then there exists a sequence of sets of functions $\mathcal{F}_n=\{f_n^{(j)}\}_{j\in \{0,1\}^n}$, where $f_n^{(j)}:\{0,1\}^n \rightarrow \{0,1\}^n$, such that:

        \begin{enumerate}[$(i)$]
            \item All functions $f_n^{(j)}$ can be efficiently computed by a classical algorithm, i.e., there exists a classical algorithm such that given $j$ and $x\in \{0,1\}^n$ outputs $f_n^{(j)}(x)$ in polynomial time.
        
            \item The sequence $\mathcal{F}_n$ passes all polynomial-time statistical test for functions.
        \end{enumerate}

    \end{theorem}

    To understand this theorem, we first need to introduce the notion of a polynomial-time statistical test.
    
    \begin{definition}
        \cite{goldreich1986construct} A polynomial-time statistical test for functions is a probabilistic polynomial-time algorithm $T$ that, given $n$ as input and access to an oracle $O_{f_n}$ for a function $f_n:\{0,1\}^n \rightarrow \{0,1\}^n$, outputs either 0 or 1. Algorithm $T$ can query the oracle $O_{f_n}$ only by writing on a special query tape some $x \in \{0,1\}^n$ and will read the oracle answer $f_n(x)$ on a separate answer-tape. As usual $O_{f_n}$ prints its answer in one step.
    \end{definition}

    We say that a set of functions $\mathcal{F}_n$ passes the statistical test $T$ if, for any polynomial $q(n)$,
    \begin{equation}
        \left| \mathbb{P}_{f_n^{(j)}\in \mathcal{F}_n,\,T}\left(T(1^n,O_{f_n^{(j)}})=0 \right)-\mathbb{P}_{h_n\in \mathcal{H}_n,\,T}\left(T(1^n,O_{h_n})=0 \right)\right|\leq \frac{1}{q(n)}
    \end{equation}
    for a sufficiently large $n$. The probabilities are taken with respect to selecting a function uniformly at random from the sets $\mathcal{F}_n$ or $\mathcal{H}_n$, respectively, as well as over the internal randomness of the test $T$. Here, $\mathcal{H}_n$ denotes the set of all functions $h:\{0,1\}^n\rightarrow \{0,1\}^n$.

    \begin{corollary}\label{corollary_distributions}
    For any efficient algorithm $\mathcal{A}$, polynomials $p(n,m)$ and $r(n)$, and efficiently samplable sequences of distributions $D^E$ and $D^T$, it holds that for any polynomial $q(n)$,
    \begin{align}
        &\Big| \mathbb{P}_{f_n^{(j)}\in \mathcal{F}_n, x, \,z^{p(n,r(n))},\,\mathcal{A}}\left(\mathcal{A}(x,\mathcal{T}_{p(n,r(n))}^{f^{(j)}_n},1^{r(n)})=[f_n^{(j)}(x)]_1 \right)\nonumber \\& \hspace{3cm}-\mathbb{P}_{h_n\in \mathcal{H}_n,x,\,z^{p(n,r(n))},\,\mathcal{A}}\left(\mathcal{A}(x,\mathcal{T}_{p(n,r(n))}^{h_n},1^{r(n)})=[h_n(x)]_1 \right)\Big|\leq \frac{1}{q(n)}
    \end{align}
    for a sufficiently large $n$, where $x\sim D_n^E$, $z^{p(n,r(n))}$ denotes ${p(n,r(n))}$ i.i.d. samples from $D_n^T$, and $\mathcal{T}_{p(n,r(n))}^{f^{(j)}_n}=\{(z_i,f_n^{(j)}(z_i))\}_{i=1}^{p(n,r(n))}$. As before, the notation $f_n^{(j)}\in \mathcal{F}_n$ and $h_n\in \mathcal{H}_n$ indicates that the probabilities are taken with respect to selecting a function uniformly at random from the sets $\mathcal{F}_n$ or $\mathcal{H}_n$, respectively, in addition to the other random variables.
    \end{corollary}

    \begin{proof}
    Let $T$ be the following polynomial-time statistical test:
    
    \begin{enumerate}[(1)]
        \item Take $p(n, r(n))$ samples from the distribution $D^T_n$, and one sample from $D^E_n$.
        \item Using the oracle, obtain the value of the function at those positions.
        \item Compute $\mathcal{A}\bigl(x, \mathcal{T}_{p(n, r(n))}^{f_n^{(j)}}, 1^{r(n)}\bigr)\oplus [f_n(x)]_1,$ where $[f_n(x)]_1$ denotes the first bit of the output of the function $f_n$, obtained using the oracle, and $\oplus$ denotes the XOR operation.
    \end{enumerate}

        Note that since the queries are generated using efficiently samplable distributions, for any polynomial $p(n,m)$ and any efficient algorithm $\mathcal{A}$, the test runs in polynomial time. Consequently, for any polynomial $q(n)$
        \begin{align}
            &\Big| \mathbb{P}_{f_n^{(j)}\in \mathcal{F}_n, x, \,z^{p(n,r(n))},\,\mathcal{A}}\left(\mathcal{A}(x,\mathcal{T}_{p(n,r(n))}^{f^{(j)}_n},1^{r(n)})=[f_n^{(j)}(x)]_1 \right)\nonumber \\& \hspace{3cm}-\mathbb{P}_{h_n\in \mathcal{H}_n,x,\,z^{p(n,r(n))},\,\mathcal{A}}\left(\mathcal{A}(x,\mathcal{T}_{p(n,r(n))}^{h_n},1^{r(n)})=[h_n(x)]_1 \right)\Big|\leq \frac{1}{q(n)}
        \end{align}
        holds for a sufficiently large $n$.
    \end{proof}

    To proceed, the following auxiliary result is stated.

    \begin{lemma}\label{lemma_auxiliary}
    Let $D^{T}_n$ and $D^{E}_n$ be equal to the uniform distribution $U_n$. Then, for any algorithm $\mathcal{A}$ and any polynomials $p(n,m)$ and $r(n)$, it holds that, for the concept class $\mathcal{W}_n$ of all Boolean functions $w_n : \{0,1\}^n \rightarrow \{0,1\}$, 
    \begin{equation}
        \left |\mathbb{P}_{w_n\in \mathcal{W}_n,x,\,z^{p(n,r(n))},\,\mathcal{A}}\left(\mathcal{A}(x,\mathcal{T}_{p(n,r(n))}^{w_n},1^{r(n)})=w_n(x) \right)-\frac{1}{2} \right|\leq \epsilon (n)
    \end{equation}
    where $\epsilon(n)\rightarrow 0$ as $n\rightarrow \infty$. The probability is taken with respect to selecting a concept uniformly at random from the set $\mathcal{W}_n$, the internal randomness of the algorithm $\mathcal{A}$, the random variable $x \sim D_n^{E}$, and the random variable $z^{p(n,r(n))}$, which denotes $p(n,r(n))$ i.i.d. samples from $D_n^{T}$.
    \end{lemma}

    \begin{proof}
        Let $E := \{x \notin \{z_1,\dots,z_p\}\}$. Since $x$ and $z^{p}$ are independent
        and uniform, the union bound gives
        \begin{equation}
          \mathbb{P}_{x,\, z^{p}}(\overline{E})
          \le \sum_{i=1}^{p} \mathbb{P}_{x,\, z_i}(x = z_i)
          = \frac{p}{2^n}
        \end{equation}
        Now fix any realization of $(x, z^{p})$ in $E$. Because $w_n$ is uniform over $\mathcal{W}_n$ and independent of $(x, z^{p})$, the values of $w_n$ at distinct points are independent uniform bits. Hence $w_n(x)$ is uniform and independent of $\{w_n(z_i)\}_{i=1}^{p}$. The output of $\mathcal{A}$ is a function of $x$, $\mathcal{T}^{w_n}_{p}$, $1^{r(n)}$, and the internal randomness of $\mathcal{A}$. Consequently, provided that $(x,z^p)\in E$, the output is independent of $w_n(x)$. Therefore,
        \begin{equation}
          \mathbb{P}_{w_n \in \mathcal{W}_n,\, \mathcal{A}}
          \!\left( \mathcal{A}\left(x, \mathcal{T}^{w_n}_{p}, 1^{r(n)}\right) = w_n(x)
          \,\middle|\, (x, z^{p}) \right) = \frac{1}{2}
          \quad \text{for every } (x, z^{p}) \in E
        \end{equation}
        Let $C$ be the success event in the statement of the lemma, then
        \begin{equation}
          \frac{1}{2}-\frac{p}{2^{n+1}}\leq \mathbb{P}(C|E)\mathbb{P}(E)\leq \mathbb{P}(C)\leq \mathbb{P}(C|E)+ \mathbb{P}(\overline{E})\leq \frac{1}{2}+\frac{p}{2^n}
        \end{equation}
    \end{proof}

    Finally, we combine these results to prove Lemma \ref{lemma_3}.

    \begin{proof}
        Using Corollary \ref{corollary_distributions}, we have that for $D^T$ and $D^E$ being equal to the sequence of uniform distributions, 
        \begin{align}
            &\Big| \mathbb{P}_{f_n^{(j)}\in \mathcal{F}_n, x, \,z^{p(n,r(n))},\,\mathcal{A}}\left(\mathcal{A}(x,\mathcal{T}_{p(n,r(n))}^{f^{(j)}_n},1^{r(n)})=[f_n^{(j)}(x)]_1 \right)\nonumber \\& \hspace{3cm}-\mathbb{P}_{h_n\in \mathcal{H}_n,x,\,z^{p(n,r(n))},\,\mathcal{A}}\left(\mathcal{A}(x,\mathcal{T}_{p(n,r(n))}^{h_n},1^{r(n)})=[h_n(x)]_1 \right)\Big|\leq \frac{1}{q(n)}
        \end{align}
        Note that the second term coincides with the one analyzed in Lemma \ref{lemma_auxiliary}, that is, the concept class consists of all possible Boolean functions $\{0,1\}^n \rightarrow \{0,1\}$. Consequently,  
        \begin{equation}
            \mathbb{P}_{f_n^{(j)}\in \mathcal{F}_n, x, \,z^{p(n,r(n))},\,\mathcal{A}}\left(\mathcal{A}(x,\mathcal{T}_{p(n,r(n))}^{f^{(j)}_n},1^{r(n)})=[f_n^{(j)}(x)]_1 \right)\leq \frac{1}{2}+\epsilon(n)+ \frac{1}{q(n)}
        \end{equation}
        for any polynomials $p(n,m)$ and $r(n)$, and any efficient algorithm $\mathcal{A}$. Therefore, the concept class $\mathcal{C}_n := \{ [f_n^{(j)}]_1\}_{j\in \{0,1\}^n}$ is not learnable for the test distribution $D^E=U$ when $U$ is also the distribution of the training set.
    \end{proof}

\end{appendices}

\end{document}